\documentclass{article}
\usepackage{arxiv}

\usepackage[utf8]{inputenc} %
\usepackage[T1]{fontenc}    %
\usepackage{hyperref}       %
\usepackage{url}            %
\usepackage{booktabs}       %
\usepackage{amsfonts}       %
\usepackage{nicefrac}       %
\usepackage{microtype}      %

\usepackage{amsmath,amsfonts,bm}

\def\eqref#1{equation~\ref{#1}}
\def\Eqref#1{Equation~\ref{#1}}

\def\1{\bm{1}}

\def\rw{{\textnormal{w}}}
\def\rx{{\textnormal{x}}}
\def\ry{{\textnormal{y}}}

\def\rdelta{{\delta}}
\def\rtheta{{\theta}}

\def\rvw{{\mathbf{w}}}
\def\rvx{{\mathbf{x}}}
\def\rvy{{\mathbf{y}}}

\def\rvdelta{{\bm{\delta}}}
\def\rvzeta{{\bm{\zeta}}}
\def\rvtheta{{\bm{\theta}}}
\def\rvvarphi{{\bm{\varphi}}}
\def\rveta{{\bm{\eta}}}
\def\rvxi{{\bm{\xi}}}
\def\rvrho{{\bm{\rho}}}
\def\rvpsi{{\bm{\psi}}}

\def\ervdelta{{\delta}}
\def\ervzeta{{\zeta}}
\def\ervtheta{{\theta}}
\def\ervvarphi{{\varphi}}
\def\erveta{{\eta}}
\def\ervxi{{\xi}}

\def\vzero{{\bm{0}}}
\def\vone{{\bm{1}}}

\def\vtheta{{\bm{\theta}}}
\def\va{{\bm{a}}}

\def\vm{{\bm{m}}}

\def\vr{{\bm{r}}}

\def\vv{{\bm{v}}}

\def\vx{{\bm{x}}}
\def\vy{{\bm{y}}}

\def\evm{{m}}

\def\mA{{\bm{A}}}
\def\mB{{\bm{B}}}

\def\mI{{\bm{I}}}

\def\mR{{\bm{R}}}

\DeclareMathAlphabet{\mathsfit}{\encodingdefault}{\sfdefault}{m}{sl}
\SetMathAlphabet{\mathsfit}{bold}{\encodingdefault}{\sfdefault}{bx}{n}

\def\emB{{B}}

\newcommand{\E}{\mathbb{E}}

\newcommand{\R}{\mathbb{R}}

\DeclareMathOperator{\Tr}{Tr}

 \usepackage{amsmath,amssymb,amsthm}
\usepackage{color}
\usepackage{graphicx}
\usepackage{natbib}
\usepackage{makecell}
\newtheorem{theorem}{Theorem}
\newtheorem{lemma}{Lemma}

\newtheorem{assumption}{Assumption}

\newtheorem{corollary}{Corollary}

\newtheorem{definition}{Definition}

\def\diag{\mathop{\rm diag}\nolimits}

\title{Bayesian interpretation of SGD as It\^o process}

\author{
  Soma Yokoi \\
  The University of Toyko / RIKEN\\
   \And
  Issei Sato\\
  The University of Toyko / RIKEN\\
}

\begin{document}
\maketitle

\begin{abstract}
The current interpretation of stochastic gradient descent (SGD) as a stochastic process lacks generality in that its numerical scheme restricts continuous-time dynamics as well as the loss function and the distribution of gradient noise.
We introduce a simplified scheme with milder conditions that flexibly interprets SGD as a discrete-time approximation of an It\^o process.
The scheme also works as a common foundation of SGD and stochastic gradient Langevin dynamics (SGLD), providing insights into their asymptotic properties.
We investigate the convergence of SGD with biased gradient in terms of the equilibrium mode and the overestimation problem of the second moment of SGLD.
\end{abstract}

\section{Introduction}
There is a demand to clarify statistical properties of stochastic gradient descent (SGD) by interpreting as a stochastic process, based on stochastic gradient-Markov chain Monte Carlo (SG-MCMC).
The recent studies have achieved it by considering SGD as a limited variant of stochastic gradient Langevin dynamics (SGLD)~\citep{Welling:2011aa}, but at the expense of requiring strong assumptions on the gradient noise and loss function of SGD, which also result in a restricted parameter distribution \citep{Mandt:2017aa, Maddox:2019aa}.
In this paper, we solve these problems simultaneously with a new numerical scheme that approximates an It\^o process with milder conditions.
We validate it by a weak approximation theory in stochastic analysis and numerical experiments.

To the best of our knowledge, this paper is the first to introduce a simplified numerical scheme that has fewer conditions than the Euler-Maruyama~(EM) scheme in the SGD and SG-MCMC contexts.
The first order EM scheme has been used in SGLD and extensions \citep{Welling:2011aa, Ma:2015aa}, and higher order schemes have been used for extra sampling accuracy \citep{Chen:2015ab,Mou:2019vl}.
In contrast, we introduce a lower order scheme to construct a common foundation of SGD and SGLD, bridging the gap between them.
Our scheme adopts a somewhat inaccurate discrete-time approximation of Wiener processes.
This reduces the assumptions about gradient noise and allows a wider class of loss functions and parameter distributions.
Technically, our scheme is shown to be the most fundamental and least hypothetical interpretation of SGD in the weak approximation aspect.

Our new interpretation contribute to simplifying the understanding of SGD.
Many optimization problems by SGD with smooth loss functions can be interpreted as It\^o processes.
It gives insight into the asymptotic behavior of SGD and SGLD.
As examples, we investigate the effect of reducing stepsize on convergence with biased gradient in terms of stationary distribution.
We also reinterpret SGLD and propose an extension to avoid overestimating the second moment of diffusion.

\begin{table*}[t!]
\caption{Comparison of Numerical Schemes}\label{tab:comp_moment}
\begin{center}
\begin{tabular}{lccccc}
\toprule
    &\textbf{Skewed EM} &\textbf{Simplified EM} &\textbf{EM}\\
\midrule
\makecell*{$p$-th moment\\condition} 
& \makecell*{$
    \begin{cases}
        0   & (p = 1)\\
        1   & (p = 2)
    \end{cases}$}
& \makecell*{$
    \begin{cases}
        0   & (p = 1,3)\\
        1   & (p = 2)
    \end{cases}$}
& \makecell*{$
    \begin{cases}
        0       & (p = 1,3,\dots)\\
        (p-1)!! & (p = 2,4,\dots)
    \end{cases}$}\\
\makecell*{Weak order} & $0.5$ & $1$ & $1$\\
\bottomrule
\end{tabular}
\end{center}\end{table*}

Our contributions are summarized as follows:
\begin{itemize}
    \item Section\,\ref{sec:scheme}: we introduce a new numerical scheme of continuous-time dynamics that has the fewest requirements of the diffusion moment.
    \item Section\,\ref{sec:sampling}: we propose two Bayesian sampling methods based on the schemes introduced in Section\,\ref{sec:scheme}, ensuring the Bayesian posterior for the stationary distribution.
    \item Section\,\ref{sec:SGD}: we present a novel interpretation of SGD based on the new scheme, demonstrating generality in that it eliminates the most assumptions of the previous studies.
    \item Section\,\ref{sec:reinterpret_SGLD}: we reinterpret SGLD as the new scheme of an It\^o process and propose an extension to avoid overestimating the second moment of diffusion.
\end{itemize}

\section{Weak numerical scheme of SDE} \label{sec:scheme}
In this section, we quickly review the existing numerical schemes of the It\^o process.
Then we propose a new scheme called skewed EM.

Let us consider continuous-time dynamics of a stochastic differential equation (SDE)
\begin{equation}\label{equ:general_SDE}
    d\rvx_t = \va(\rvx_t)dt + \mB(\rvx_t)d\rvw_t,
\end{equation}
where $\rvx_t$ is the $\R^D$-stochastic process with given initial value $\rvx_{t_0}$, $\rvw_t$ is the $\R^M$-standard Wiener process, and coefficients $\va(\vx)\in\R^D$, $\mB(\vx)\in\R^{D\times M}$.
Subsequent discussion is subject to the assumption common to previous studies e.g., \citep{Sato:2014xy}.
\begin{assumption}\label{asm:common}\leavevmode
    \begin{enumerate}
        \item Coefficient functions $\va(\vx)$, $\mB(\vx)$ are sufficiently smooth and satisfy a Lipschitz condition:
    for all $\vx,\vy\in\R^D$ the following inequality holds for some constant $C>0$
    \begin{equation}\label{equ:Lipschitz}
        \begin{aligned}
            \|\va(\vx)-\va(\vy)\|_2 + \sum_{m=1}^M &\|\emB^{*,m}(\vx)-\emB^{*,m}(\vy)\|_2 \leq C\|\vx-\vy\|_2,
        \end{aligned}
    \end{equation}
    where $*$ indicates all indexes.
        \item Test function $f:\R^D\to\R$ is sufficiently smooth.
    \end{enumerate}
\end{assumption}
We also employ the assumption of \citet{Milstein:1986aa,Milstein:1995aa}, which is $\kappa$-th growth condition of partial derivatives.
\begin{assumption}\label{asm:milstein}
    Functions $\va(\vx)$, $\mB(\vx)$, $f$, and their partial derivatives up to the third orders, satisfy $\kappa$-th growth condition: for function $g$ there exists some constants $C>0, \kappa\geq0$ such that
    \begin{equation}
        \|g(\vx)\|_2 \leq C (1+\|\vx\|_2^\kappa).
    \end{equation}
\end{assumption}

\subsection{Review: EM scheme}
Consider time interval $[t_0, t_0+T]$ with step $\epsilon = T/K: t_0<t_1<\cdots<t_K=t_0+T$, $t_{k+1}-t_k = \epsilon$.
The EM scheme of \autoref{equ:general_SDE} is known as
\begin{equation}\label{equ:EM}
    \rvy_{k+1} = \rvy_k + \epsilon \va(\rvy_k) + \sqrt{\epsilon} \mB(\rvy_k) \rvzeta_k,
\end{equation}
where $\rvy\in\R^D$ has the same initial value $\rvy_0 = \rvx_{t_0}$ and $\rvzeta$ follows the $\R^M$-standard Gaussian distribution $\rvzeta\sim\mathcal{N}(\vzero,\mI_M)$.

The approximation accuracy of the EM scheme has been widely investigated \citep{Milstein:1979aa,Talay:1984aa,Kloeden:1992aa}.
The EM scheme has order of accuracy $1$ in weak (distribution law) approximations:
\begin{definition}[Weak convergence]
    The time-discrete approximation $\{\rvy_k\}_{k=1}^K$ has weak order $q>0$ if there exists some constant $C>0$ such that
    \begin{equation}
        \left| \E[f(\rvx_{t_k})] - \E[f(\rvy_k)] \right| \leq C \epsilon^q
    \end{equation}
    for $0<\epsilon<1$, all $K>0$, $k=0,1,\dots,K$, and some test function $f$.
\end{definition}

The Gaussian constraint of the diffusion $\rvzeta$ can be written in a different form.
It is known that a standard Gaussian random variable $\ervzeta^d\in\R$ satisfies the following moment condition for $p=1,2,\dots,\infty$:
\begin{equation} \label{equ:SGLD_moments}
    \mathbb{E}[(\ervzeta^d)^p] = \begin{cases}
        0       & (p = 1,3,\dots)\\
        (p-1)!! & (p = 2,4,\dots),
    \end{cases}
\end{equation}
where $!!$ is the double factorial.
That is, the component $\ervzeta^d$ must satisfy an infinite number of the moment conditions.
In this view, we show in this section how we can reduce the requirement in numerical schemes.

\begin{figure*}[t!]\begin{center}
  \includegraphics[width=0.32\linewidth]{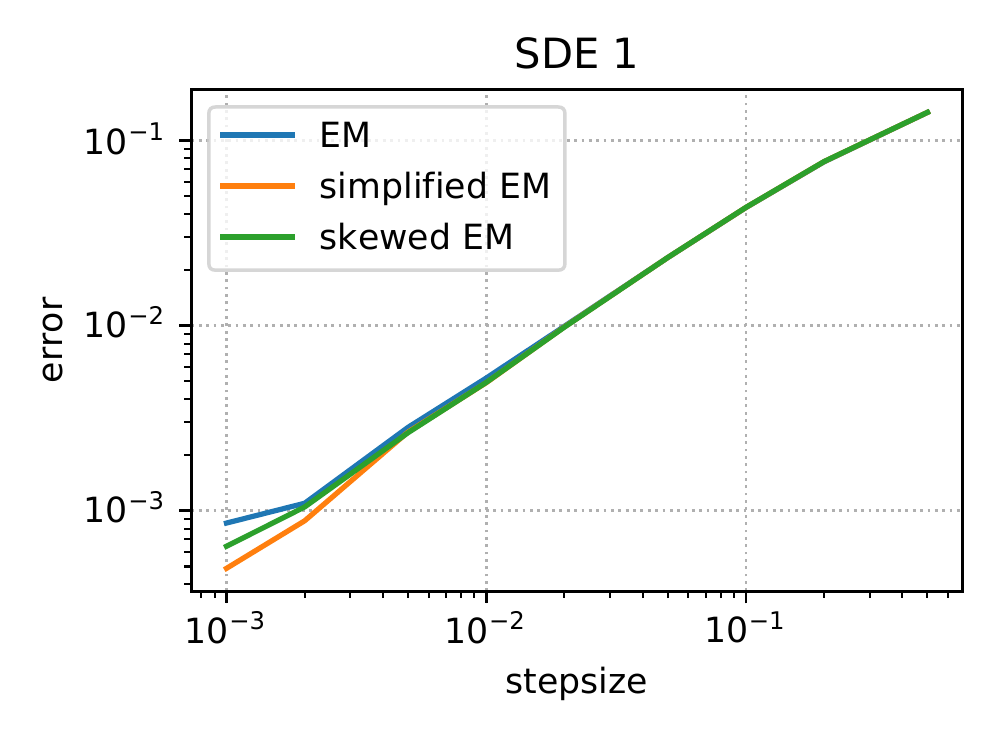}
  \includegraphics[width=0.32\linewidth]{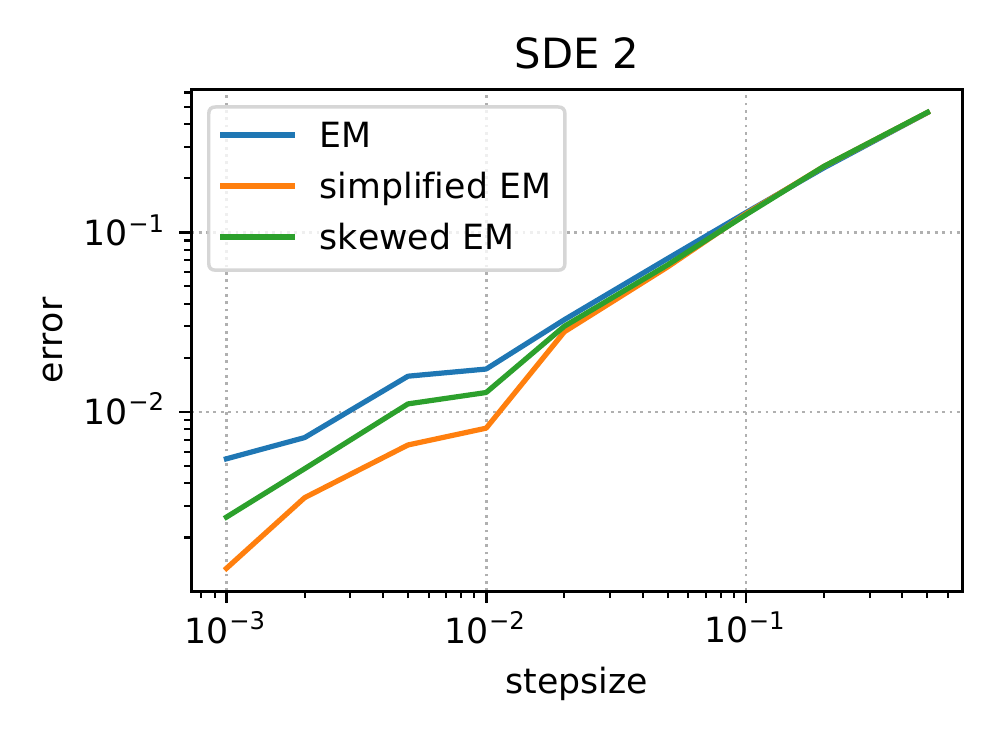}
  \includegraphics[width=0.32\linewidth]{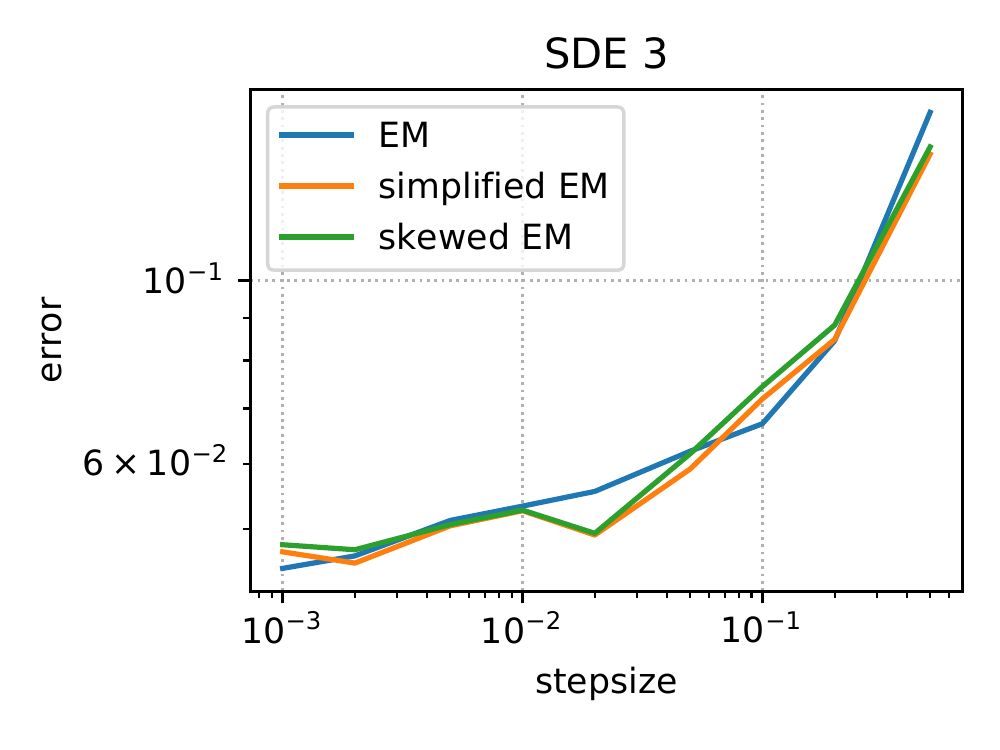}
  \caption{
    Mean error of the numerical schemes.
  }
  \label{fig:example_scheme}
\end{center}\end{figure*}

\subsection{Review: simplified EM scheme}
The simplified EM scheme is a less restrictive than the EM scheme, although it achieves the same accuracy in weak approximations.
It was first introduced by \citet{Milstein:1979aa} and \citet{Talay:1984aa} to demonstrate a non-Gaussian time-discrete approximation of the Wiener process.
The simplified EM scheme of \autoref{equ:general_SDE} can be written as
\begin{equation}\label{equ:alg_simplifiedEM}
  \rvy_{k+1} = \rvy_k + \epsilon \va(\rvy_k) + \sqrt{\epsilon} \mB(\rvy_k) \rveta_k,
\end{equation}
where $\rveta = (\erveta^1, \dots, \erveta^D)^\top$ is a random vector satisfying the following moment condition for $p=1,2,3$ and $d=1,\dots,D$,
\begin{equation}\label{equ:moment_simpleEM}
    \mathbb{E}[(\erveta^d)^p] = \begin{cases}
        0   & (p = 1,3)\\
        1   & (p = 2).
    \end{cases}
\end{equation}
The fourth or higher moments are not constrained.

The common use of this scheme is replacing the Gaussian random variable of the EM scheme with a symmetric two-point distributed random variable
\begin{equation}
    P(\erveta^d = \pm \sqrt{\epsilon}) = \frac{1}{2}
\end{equation}
to reduce computation time in pseudorandom number generation \citep{Kloeden:1992aa}.

Note that the simplified EM scheme enjoys sampling accuracy of weak order $1$, which is the same as the EM scheme.
The higher moment constraints ($p\geq 4$) does not contribute to accuracy in this respect.

\subsection{Proposal: skewed EM scheme}
We introduce a new numerical scheme, skewed EM, as the simplest EM scheme.
The following theorem is the foundation of this paper.

\begin{theorem}\label{thm:main}
    Suppose Assumption\,\ref{asm:common} and \ref{asm:milstein}.
    Consider time-discrete approximation
    \begin{equation}\label{equ:skewedEM}
        \rvy_{k+1} = \rvy_k + \epsilon \va(\rvy_k) + \sqrt{\epsilon} \mB(\rvy_k) \rvxi_k,
    \end{equation}
    where $\rvxi\in\R^M$ is a random vector with component $\ervxi^m, \ m=1,\dots,M$ satisfying moment condition
    \begin{equation}\label{equ:moment_skewEM}
        \mathbb{E}[(\ervxi^m)^p] = \begin{cases}
            0   & (p = 1)\\
            1   & (p = 2).
        \end{cases}
    \end{equation}
    Then, for all $K>0$ and all $k=0,1,\dots,K$ the following inequality holds:
    \begin{equation}
        \left| \E[f(\rvx_{t_k})] - \E[f(\rvy_k)] \right| \leq C \sqrt{\epsilon},
    \end{equation}
    where $C>0$ is some constant.
    i.e. \Eqref{equ:skewedEM} has order of accuracy $0.5$ in the sense of weak approximations.
\end{theorem}
\begin{proof}
    See \autoref{app:proof_thm:main}.
\end{proof}

We call this scheme skewed EM, because the third moment is not standardized.
Note that the skewed EM scheme has the fewest moment conditions, at the cost of a slower weak order than the EM and simplified EM schemes.
We will see later that eliminating the third or higher moment conditions contributes to simplifying the interpretation of SGD.

The skewed EM scheme is the simplest possible EM scheme.
The following lemma implies that the moment conditions cannot be fewer than \Eqref{equ:moment_skewEM}.
\begin{lemma}\label{lem:impossibleEM}
    Suppose Assumption\,\ref{asm:common} and \ref{asm:milstein}.
    Consider time-discrete approximation
    \begin{equation}\label{equ:impossibleEM}
        \rvy_{k+1} = \rvy_k + \epsilon \va(\rvy_k) + \sqrt{\epsilon} \mB(\rvy_k) \rvpsi_k,
    \end{equation}
    where $\E[\rvpsi] = \vzero$.
    That is, only the first moment is constrained.
    Then, the following inequality holds:
    \begin{equation}
        \left| \E[f(\rvx_{t_k})] - \E[f(\rvy_k)] \right| \leq C,
    \end{equation}
    i.e. \Eqref{equ:impossibleEM} does not weakly converge (order $0$).
\end{lemma}
\begin{proof}
    See \autoref{app:proof_lem:impossibleEM}.
\end{proof}

\begin{figure*}[t!]\begin{center}
  \includegraphics[width=0.32\linewidth]{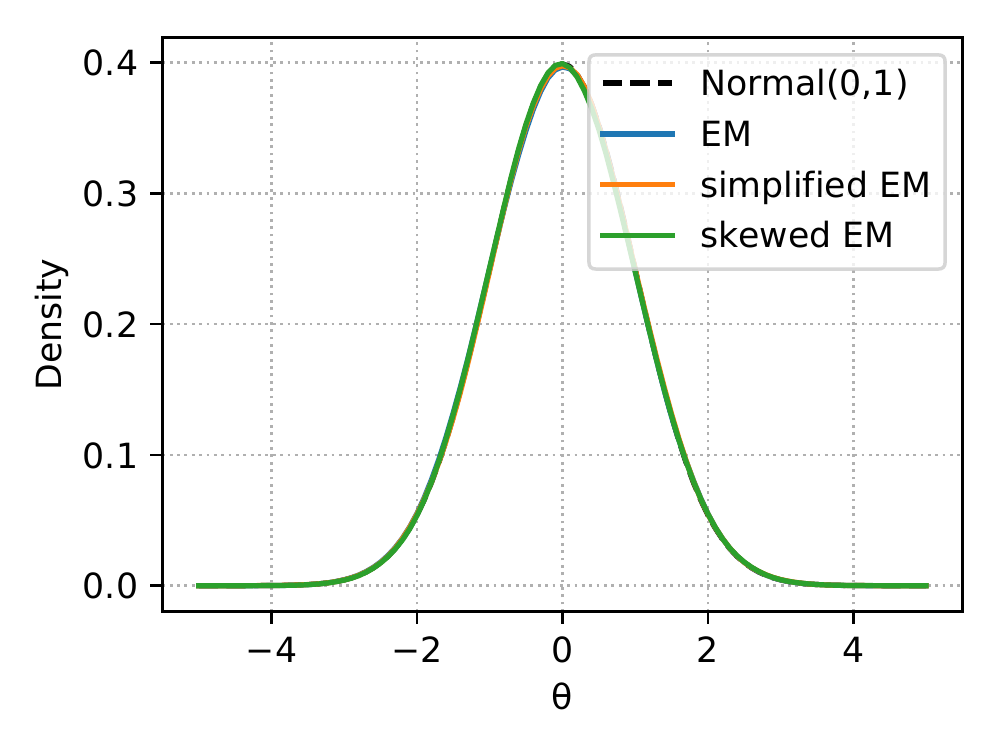}
  \includegraphics[width=0.32\linewidth]{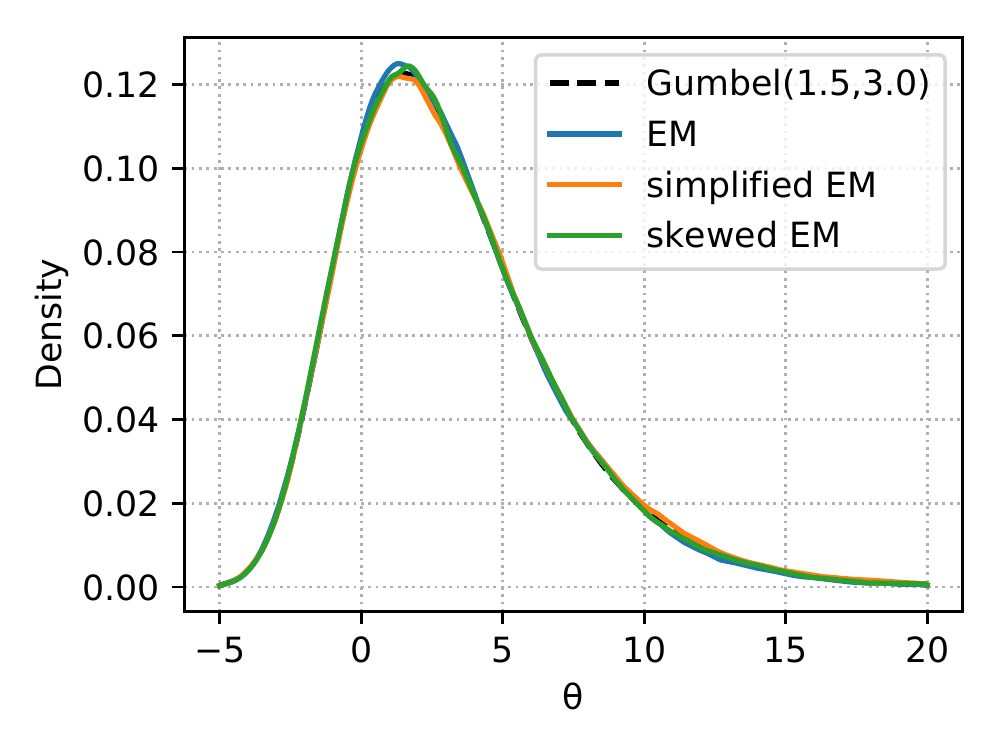}
  \includegraphics[width=0.32\linewidth]{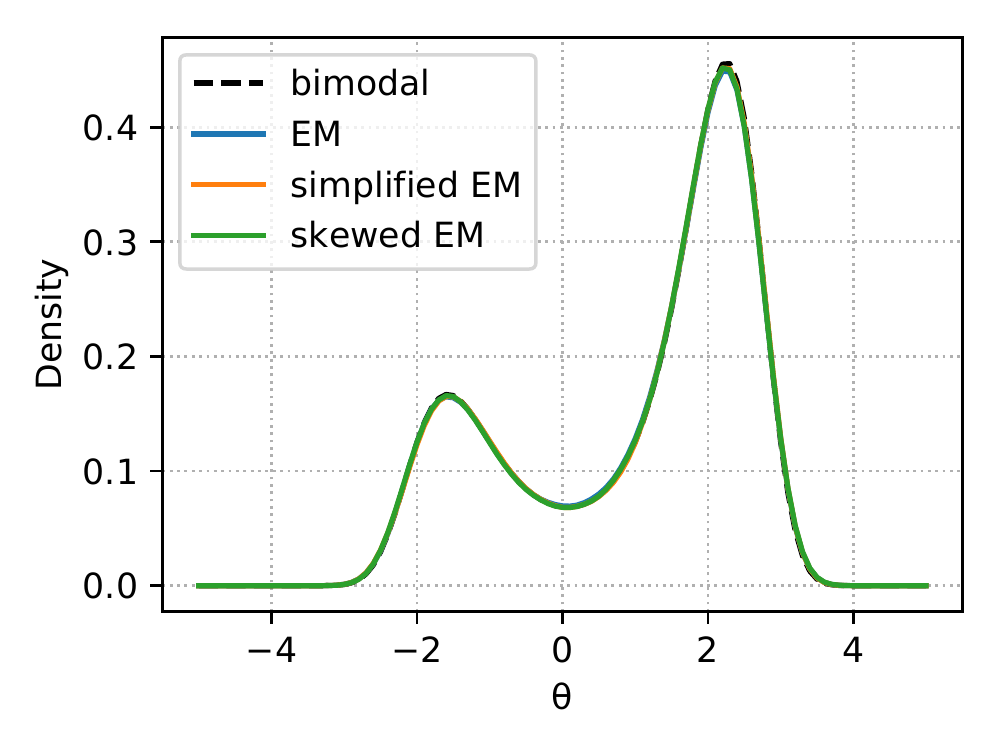}
  \caption{
    Sampling results of ULA with the numerical schemes.
  }
  \label{fig:example_sampling}
\end{center}\end{figure*}

\subsection{Example}
\autoref{fig:example_scheme} shows the mean error of the EM, simplified EM, and skewed EM schemes.
We observed that small stepsize improves the error almost linearly, and there was no significant difference between the skewed EM scheme and the other weak order $1$ schemes in these simple examples.

The noise was generated by, for the simplified EM scheme,
\begin{equation}
    P\left(\erveta^d=+1\right) = 0.5, \ \ \
    P\left(\erveta^d=-1\right) = 0.5.
\end{equation}
and for the skewed EM scheme,
\begin{equation}
    P\left(\ervxi^d=\sqrt{\frac{3}{2}}\right) = 0.4, \ \ \
    P\left(\ervxi^d=-\sqrt{\frac{2}{3}}\right) = 0.6.
\end{equation}

The error is given by
\begin{equation}
    \left| \E[\rx_T] - \E[\ry_K] \right|,
\end{equation}
where $t_0=0$, $T=1$, and $K=T/\epsilon$.
The expected value was approximated by 10,000 samples for each setting.
The target SDEs are explicitly solvable ones:
\begin{enumerate}
\item Let $\rx\in\R$, $\rx_0 = 0.1$, and $\rvw\in\R^2$, 
\begin{equation}
    d\rx_t = \frac{3}{2}\rx_t dt + \left(\frac{1}{10}\rx_t \ \frac{1}{10}\rx_t\right) d\rvw_t.
\end{equation}
\item Let $\rx\in\R$, $\rx_0 = -1.0$, and $\rvw\in\R$,
\begin{equation}
    d\rx_t = (\rx_t+2)dt + 3d\rw_t.
\end{equation}
\item Let $\rx\in\R$, $\rx_0 = -1.0$, and $\rvw\in\R$,
\begin{equation}
    d\rx = -\mathrm{tanh}(\rx)\left(1+\frac{1}{2}\mathrm{sech}^2(\rx)\right)dt + \mathrm{sech}(\rx)d\rw_t.
\end{equation}
\end{enumerate}

\section{Bayesian sampling by weak schemes}\label{sec:sampling}
In this section, we introduce three Bayesian sampling methods that correspond to the EM, simplified EM, skewed EM schemes of the Langevin dynamics.
One results in an existing algorithm and the others are novel in that they adapt to non-Gaussian noise as the diffusion, guaranteeing convergence to the target distribution without Metropolis-Hastings rejection.

Let us consider continuous-time Langevin dynamics
\begin{equation} \label{equ:LD_process}
    d\rvtheta_t = \nabla \log \pi(\rvtheta_t) dt + \sqrt{2} d\rvw_t,
\end{equation}
where $\rvtheta_t$ is the $\R^D$-stochastic process, $\rvw_t$ is the $\R^D$-standard Wiener process, and $\pi(\rvtheta)$ is a target joint probability, e.g., Bayesian posterior.
Let $p(\rvtheta)$ be a stationary distribution of the process $\rvtheta_t$.
From the Fokker--Planck equation, the steady state is known to follow the target distribution
\begin{equation}
    p(\rvtheta) = \pi(\rvtheta).
\end{equation}
One can derive a discrete-time algorithm whose samples follow the target distribution
\begin{equation}
    \{\rvtheta_k\}_{k=1}^K \sim \pi(\rvtheta)
\end{equation}
by applying the numerical schemes to \autoref{equ:LD_process}.

\subsection{Review: sampling by EM scheme}
The well known unadjusted Langevin algorithm~(ULA) is the EM scheme of \autoref{equ:LD_process}
\begin{equation}\label{equ:alg_ULA}
  \rvtheta_{k+1} = \rvtheta_k + \epsilon \nabla \log \pi(\rvtheta_k) + \sqrt{2\epsilon} \rvzeta_k,
\end{equation}
where $\rvzeta_k\sim\mathcal{N}(\vzero,\mI_D)$.
The SGLD algorithm uses stochastic gradient $\widehat{\nabla} \log \pi(\rvtheta_k)$ instead of exact gradient $\nabla \log \pi(\rvtheta_k)$ to avoid iteration over large dataset.

According to \citet{Sato:2014xy}, the weak approximation of SGLD remains the same order $1$, however it does not conform to the strong (path-wise) approximation of the EM scheme.
The noise of stochastic gradients breaks the consistency of the EM scheme by affecting its sample path even in the infinitesimal stepsize limit.
Fortunately, this is not an issue for Bayesian sampling that is interested in the distribution of the time-discrete approximation.
In this sense, the EM scheme in SGLD can be unnecessarily strong and altered by some weak alternatives.

\subsection{Proposal: sampling by simplified and skewed scheme}
Let us consider algorithmic counterparts of the simplified and skewed EM schemes for Bayesian sampling.
By applying the schemes to \autoref{equ:LD_process}, we obtain a simplified ULA algorithm:
\begin{equation}
    \rvtheta_{k+1} = \rvtheta_k + \epsilon \nabla \log \pi(\rvtheta_k) + \sqrt{2\epsilon} \rveta_k,
\end{equation}
where
\begin{equation}
    \mathbb{E}[(\erveta^d)^p] = \begin{cases}
        0   & (p = 1,3)\\
        1   & (p = 2),
    \end{cases}
\end{equation}
and a skewed ULA algorithm:
\begin{equation}
    \rvtheta_{k+1} = \rvtheta_k + \epsilon \nabla \log \pi(\rvtheta_k) + \sqrt{2\epsilon} \rvxi_k,
\end{equation}
where
\begin{equation}
    \mathbb{E}[(\ervxi^d)^p] = \begin{cases}
        0   & (p = 1)\\
        1   & (p = 2).
    \end{cases}
\end{equation}
The weak convergence is given by \autoref{thm:main}, which guarantees sampling accuracy without a rejection step.

\subsection{Example}
Figure\,\ref{fig:example_sampling} illustrates the sampling results of the ULA algorithm, with the EM, simplified EM, and skewed EM schemes, for the Gaussian, Gumbel, and a bimodal distributions.
These are examples of asymmetric diffusion which result in accurate sampling.

The density function of the bimodal distribution was given by
\begin{equation}\label{equ:bimodal_pdf}
    \pi(x) \propto \exp \left( - \frac{(x-3)(x-1)(x+1)(x+2)}{10} \right).
\end{equation}

\begin{table*}[t!]
\caption{Comparison of SGD Interpretations as Continuous-Time Dynamics}\label{tab:comp_SGDinterpret}
\begin{center}
\begin{tabular}{lcccc}
\toprule
    &\textbf{Gradient Noise} &\textbf{Loss Function} &\textbf{Posterior} &\textbf{Process}\\
\midrule
This paper            & Finite moment & Any       & Any       & It\^o process\\
\cite{Mandt:2017aa}   & Gaussian      & Quadratic & Gaussian  & OU process\\
\cite{Maddox:2019aa}  & Gaussian      & Quadratic & Gaussian  & OU process\\
\bottomrule
\end{tabular}
\end{center}\end{table*}

\section{Bayesian interpretation of SGD}\label{sec:SGD}
In this section, we present a new Bayesian interpretation of SGD using the skewed EM scheme, demonstrating the flexibility of modeling by eliminating most assumptions of the previous studies at the same time, as shown in Table\,\ref{tab:comp_SGDinterpret}.
Whereas the gradient noise in SGD is the only source of randomness, the skewed EM scheme requires noise as a diffusion term.
We fill the gap here.

The SGD algorithm is well known as
\begin{equation}\label{equ:alg_SGD}
    \vtheta_{k+1} = \vtheta_k - \epsilon \widehat{\nabla}L(\vtheta_k),
\end{equation}
where $\widehat{\nabla}$ is stochastic gradient and $L(\vtheta)$ is a loss function, e.g., error of a softmax classifier of a neural network with L1 regularization.
If we see $L(\vtheta)$ as energy of state $\vtheta$ and consider a Gibbs distribution
\begin{equation}
    \pi(\vtheta) \propto \exp(- L(\vtheta)),
\end{equation}
which corresponds to the Bayesian posterior, it results in a stochastic optimization for the maximum a posteriori (MAP) point estimate of $\pi(\vtheta)$
\begin{equation}\label{equ:alg_SGD_Bayes}
    \rvtheta_{k+1} = \rvtheta_k + \epsilon \widehat{\nabla} \log \pi(\rvtheta_k)
\end{equation}
with gradient noise $\rvdelta$
\begin{equation} \label{equ:gradnoise}
    \rvdelta = \widehat{\nabla} \log \pi(\rvtheta) - \nabla \log \pi(\rvtheta).
\end{equation}

\subsection{Review: SGD as an Ornstein-Uhlenbeck process}
\citet{Mandt:2017aa} and \citet{Maddox:2019aa} assumed gradient noise and a loss function as follows to regard \Eqref{equ:alg_SGD_Bayes} as an EM scheme.
\begin{assumption}[\cite{Mandt:2017aa}]\label{asm:Blei_noise}
    Gradient noise $\rvdelta$ follows a Gaussian distribution $\rvdelta \sim \mathcal{N}(\vzero,\mB\mB^\top)$ where $\mB\mB^\top$ is independent of $\rvtheta$ and full rank.
\end{assumption}
\begin{assumption}[\cite{Mandt:2017aa}]\label{asm:Blei_loss}
    The loss function is approximated by a quadratic function
    \begin{equation}
        L(\vtheta) \approx \frac{1}{2} \rvtheta^\top \mA\rvtheta,
    \end{equation}
    where $\mA$ is the Hessian at the optimum and positive definite.
\end{assumption}
Under such gradient noise and loss functions, they have shown that SGD with constant stepsize $\epsilon$ can be regarded as an EM scheme of a limited class of stochastic process.
\begin{theorem}[\cite{Mandt:2017aa}]
    Let Assumptions\,\ref{asm:Blei_noise} and \ref{asm:Blei_loss} be satisfied.
    Then the SGD trajectory $\{\rvtheta_k\}_{k=1}^K$ is a weak order 1 time-discrete approximation of the multivariate Ornstein-Uhlenbeck process
    \begin{equation}
        d\rvtheta_t = - \epsilon A \rvtheta_t dt + \frac{1}{\sqrt{S}}\epsilon \mB d\rvw_t.
    \end{equation}
    The stationary distribution is a Gaussian distribution
    \begin{equation}
        p(\rvtheta) \propto \exp\left(-\frac{1}{2}\rvtheta^\top \Sigma^{-1} \rvtheta\right),
    \end{equation}
    where $\Sigma\mA + \mA\Sigma = \frac{\epsilon}{S}\mB\mB^\top$.
\end{theorem}
\citet{Maddox:2019aa} have reported that \autoref{asm:Blei_noise} and \autoref{asm:Blei_loss} are not often satisfied by SGD in real-world problems e.g., VGG-16 and PreResNet-164 networks on CIFAR-100 dataset.
For partial justification, they assumed infinite data points for \autoref{asm:Blei_noise} and extended the SGD algorithm to use stochastic weight averaging so as to capture the loss geometry for \autoref{asm:Blei_loss}.
Whereas these studies have pioneered the interpretation of SGD as continuous-time dynamics, it still lacks generality.
Moreover, the Gaussian constraint of stationary distribution remains unresolved.

\subsection{Proposal: SGD as an It\^o process}
We regard \Eqref{equ:alg_SGD_Bayes} as a skewed EM scheme under the following acceptable assumption about gradient noise.
The moment conditions cannot be less than the first and second moments (\autoref{lem:impossibleEM}).
However we show in \autoref{thm:SGD_SLA} that the first and second moments can be relaxed to take any constant values.
This helps the interpretation of SGD with various gradient noise.
\begin{assumption}\label{asm:our_noise}
    Gradient noise $\rvdelta$ is a random vector with component $\ervdelta^d$ satisfying a moment condition
    \begin{equation}
        \mathbb{E}[\rvdelta] = \vm_1, \ \ \
        \mathbb{E}[\rvdelta^{\circ 2}] = \vm_2,
    \end{equation}
    where $\circ2$ is an element-wise square and $\evm_1^d, \evm_2^d$ are some constants for $d=1,\dots,D$.
\end{assumption}
The higher moments $(p\geq3)$ are not constrained in our formulation.
There is no additional assumption on loss function than \autoref{asm:common} and \ref{asm:milstein}.
Then the SGD algorithm can be regarded as a skewed EM scheme of more general class of a stochastic process.
\begin{theorem}\label{thm:SGD_SLA}
    Let \autoref{asm:our_noise} be satisfied.
    Then the SGD trajectory $\{\rvtheta_k\}_{k=1}^K$ is a weak order 0.5 time-discrete approximation of the It\^o process
    \begin{equation}\label{equ:SLA_SGD_Ito}
        d\rvtheta_t = \left(\nabla \log \pi(\rvtheta_t) + \vm_1\right) dt + \sqrt{2} \mR d\rvw_t,
    \end{equation}
    where $\mR = \sqrt{\frac{\epsilon}{2}} \diag{\left(\vm_2-\vm_1^{\circ2}\right)}^\frac{1}{2}$.
    The stationary distribution is a modified Gibbs distribution of given loss function $L(\rvtheta)$
    \begin{equation}\label{equ:SLA_SGD_stationary}
        p(\rvtheta) \propto \pi(\rvtheta)^{\frac{\Tr{(\mR^{-2})}}{D}} \exp{\left( \vm_1^\top \mR^{-2} \rvtheta \right)}^\frac{1}{D}.
    \end{equation}
\end{theorem}
\begin{proof}
    See \autoref{app:proof_thm:SGD_SLA}.
\end{proof}

\autoref{thm:SGD_SLA} suggests that SGD can be viewed as a time-discrete approximation of corresponding It\^o process without Gaussian noise and quadratic loss \autoref{asm:Blei_noise} and \ref{asm:Blei_loss}.
Moreover, we can find a continuous-time counterpart in SDE even when gradient noise is biased and has non-standardized variance.

It is also implied that the appropriate scaling of gradient noise leads a steady state to the exact Gibbs distribution of a given loss function.
The following corollary is immediately obtained from \autoref{thm:SGD_SLA}.
\begin{corollary}\label{cor:appropriate_scaling}
    Let $m_1^d = 0$ and $m_2^d = \frac{2}{\epsilon}$ for $d=1,\dots,D$, so that the SGD trajectory $\{\rvtheta_k\}_{k=1}^K$ is a weak order 0.5 time-discrete approximation of the It\^o process
    \begin{equation}\label{equ:SLA_SGD_Ito_std}
        d\rvtheta_t = \nabla \log \pi(\rvtheta_t) dt + \sqrt{2} d\rvw_t.
    \end{equation}
    Then the stationary distribution is 
    \begin{equation}
        p(\rvtheta) \propto \pi(\rvtheta).
    \end{equation}
\end{corollary}

\begin{figure*}[t!]\begin{center}
  \includegraphics[width=0.32\linewidth]{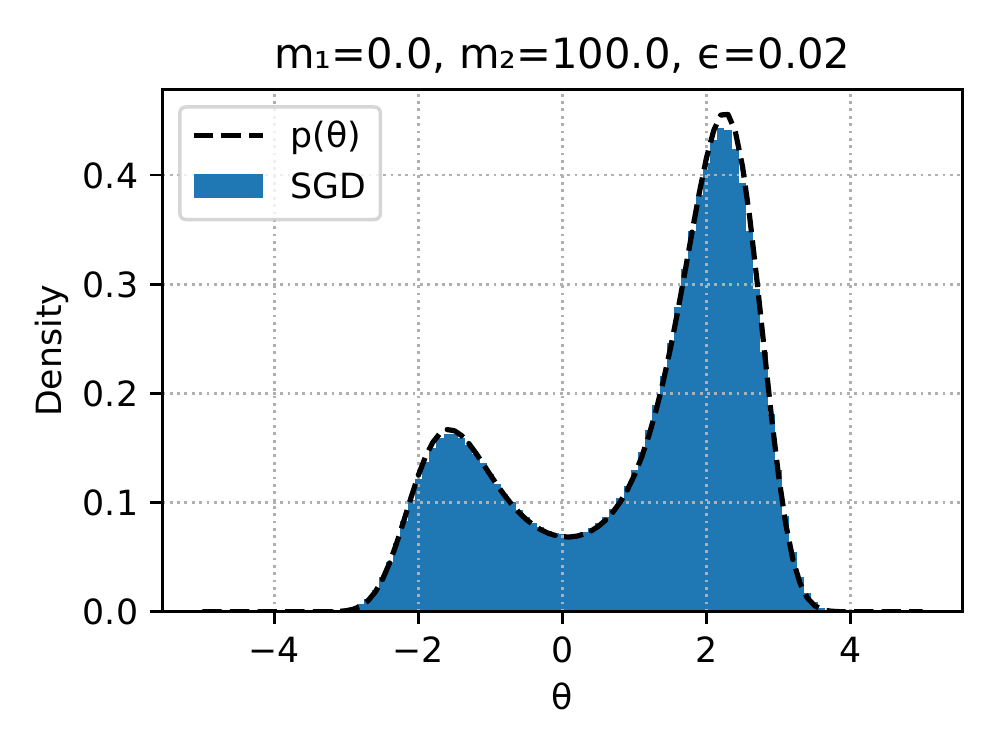}
  \includegraphics[width=0.32\linewidth]{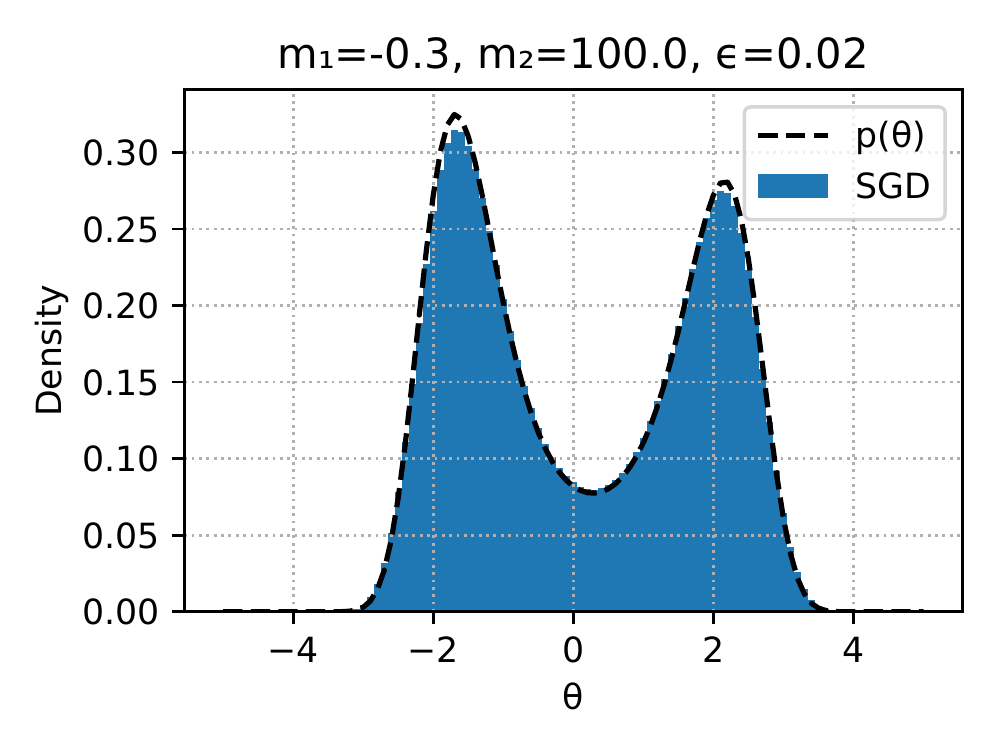}
  \includegraphics[width=0.32\linewidth]{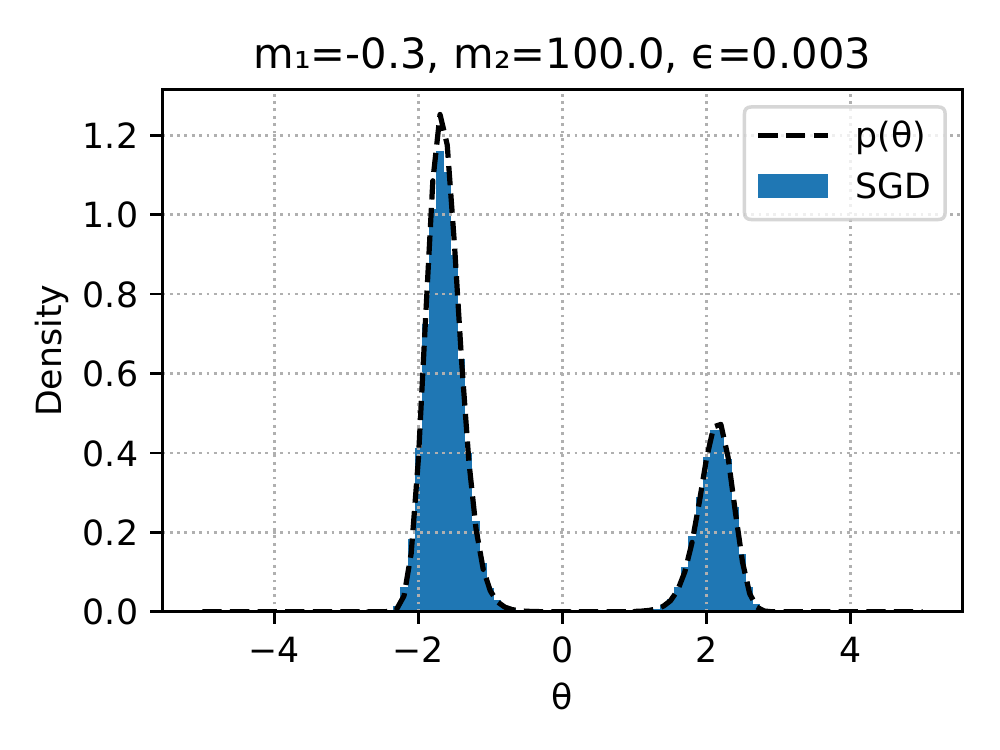}
  \caption{
    Analytic stationary distribution $p(\rvtheta)$ and histogram of SGD trajectory.
  }
  \label{fig:example_SGD}
\end{center}\end{figure*}

\subsection{Effect of stepsize}\label{subsec:effect_stepsize}
It is well known practice that by reducing stepsize, SGD converges to a certain value.
One can confirm the phenomenon from the viewpoint of the stationary distribution of the SDE using \autoref{thm:SGD_SLA}.
Let moments $m_1^d\neq0$ and $m_2^d\neq0$.
As seen in \autoref{equ:SLA_SGD_stationary}, stepsize $\epsilon>0$ affects the stationary distribution.
Small $\epsilon$ increases thermodynamic beta of the system
\begin{equation}
    \beta = {\frac{\Tr{(\mR^{-2})}}{D}} = \frac{2}{D\epsilon}\sum_{d=1}^D\frac{1}{m_2^d-(m_1^d)^2}, 
\end{equation}
so that the stationary distribution becomes sharp at its mode.
This corresponds to the convergence of the SGD optimization near the optimum.
However, small $\epsilon$ also amplifies bias
\begin{equation}\begin{aligned}
    \exp&{\left( \vm_1^\top \mR^{-2} \rvtheta \right)}^\frac{1}{D} = \exp{\left( \vm_1^\top \diag{\left(\vm_2-\vm_1^{\circ2}\right)}^{-1} \rvtheta \right)}^\frac{2}{D\epsilon}.
\end{aligned}\end{equation}
It shifts the mode of the stationary distribution apart from the loss optimum.

We cannot run the SGD algorithm for an infinite time with an infinitesimal stepsize in practice, however, we can consider it by using the time evolution of the continuous-time dynamics.
At the limit of $t\to\infty$ and $\epsilon\to 0$, the stationary distribution is the delta function centered at infinity
\begin{equation}
    \lim_{\epsilon\to0} \beta = \infty, \ \ \ \lim_{\epsilon\to0} \left|\exp{\left( \vm_1^\top \mR^{-2} \rvtheta \right)}^\frac{1}{D}\right| = \infty.
\end{equation}
These indicate pros and cons of reducing stepsize in optimization by SGD.
It is implied that SGD with small stepsize can results in a biased outcome even though the training loop is long enough.
This has a negative impact on performance.

\subsection{Example}
\autoref{fig:example_SGD} numerically illustrates \autoref{thm:SGD_SLA}.
The distribution of parameter $\rtheta$ was observed by changing gradient bias $m_1$ and stepsize $\epsilon$.
The analytic stationary distribution of \autoref{equ:SLA_SGD_stationary} was accurate and almost identical to the histogram of SGD trajectory.
It also confirms the discussion of Section\,\ref{subsec:effect_stepsize}.
Even a slight change of the gradient bias $m_1$ can significantly affect the convergence point.
Reducing the stepsize $\epsilon$ makes the distribution sharp and magnifies  biased modes.
These distributions are also examples of the highly non-Gaussian stationary distributions.

The loss function was given by 
\begin{equation}
    L(\rtheta) = -\frac{(\rtheta-3)(\rtheta-1)(\rtheta+1)(\rtheta+2)}{10}.
\end{equation}
Its Gibbs distribution $\pi(\rtheta) \propto \exp(- L(\rtheta))$ is \autoref{equ:bimodal_pdf}.
The SGD algorithm was
\begin{equation}
    \rtheta_{k+1} = \rtheta_k + \epsilon\widehat{\nabla}\log\pi(\rtheta_k),
\end{equation}
where gradient noise $\rdelta = \widehat{\nabla}\log\pi(\rtheta) - \nabla\log\pi(\rtheta)$ is
\begin{equation}
    \mathbb{E}[\rdelta] = m_1, \ \ \ 
    \mathbb{E}[\rdelta^2] = m_2.
\end{equation}

\subsection{Related work}
Many extensions have been proposed for SG-MCMC based on SGLD.
\cite{Ma:2015aa} comprehensively classified SG-MCMC algorithms, including SGLD, stochastic gradient Riemannian Langevin dynamics, and stochastic gradient Hamiltonian Monte Carlo.
These investigations are limited to the EM scheme and do not cover our formulation, though their counterparts in the skewed EM scheme can be beneficial.

There are several Bayesian interpretations of SGD other than continuous-time dynamics.
SGD and the extensions are regarded as approximate variational inference by \cite{Duvenaud:2016aa}, \cite{Chaudhari:2018aa}, \cite{Khan:2018gx}, and \cite{Zhang:2018yj}.
For an optimization aspect of SG-MCMC, \cite{Chen:2016aa} proposed a global optimization algorithm by applying simulated annealing to SGLD.

\section{Reinterpretation of SGLD}\label{sec:reinterpret_SGLD}
In this section, we reinterpret SGLD as a skewed EM scheme of an It\^o process and propose an extension to avoid overestimating the second moment of diffusion.

SGLD has been considered as an EM scheme with additional noise of stochastic gradient.
The algorithm can be written with two random sources of diffusion term $\rvzeta \sim \mathcal{N}(\vzero,\mI_D)$ and some gradient noise $\rvdelta$ as
\begin{equation}
    \rvtheta_{k+1} = \rvtheta_k + \epsilon \left(\nabla\log\pi(\rvtheta_k) + \rvdelta_k \right) + \sqrt{2\epsilon}\rvzeta_k.
\end{equation}
\citet{Sato:2014xy} assumed $\mathbb{E}[\rvdelta] = \vzero$.
\citet{Mandt:2017aa} assumed in addition that $\rvdelta\sim\mathcal{N}(\vzero,\mB\mB^\top)$.

\begin{table*}[t!]
\caption{Interpretation of SGLD}\label{tab:comp_SGLD}
\begin{center}
\begin{tabular}{lccc}
\toprule
    &\textbf{Skewed EM} &\textbf{\citet{Mandt:2017aa}} &\textbf{\citet{Sato:2014xy}}\\
\midrule
\makecell*{Noise assumption}   
& \makecell*{$\mathbb{E}[\rvdelta] = \vm_{\delta,1}$\\$\mathbb{E}[\rvdelta^{\circ 2}] = \vm_{\delta,2}$}   
& $\rvdelta\sim\mathcal{N}(\vzero,\mB\mB^\top)$   & $\E[\rvdelta]=\vzero$\\
\makecell*{Stationary distribution} & $\pi(\rvtheta)^\beta \exp{\left( \vm_1^\top \mR^{-2} \rvtheta \right)}^\frac{1}{D}$ 
& \makecell*{$\exp{(-\frac{N}{2}\rvtheta^\top\mA\rvtheta)^\beta}$}
& $\pi(\rvtheta)^\beta$\\
\makecell*{Thermodynamic beta}        
& \makecell*{$\beta\approx1$\\using adaptive diffusion}
& $\beta<1$                 & $\beta<1$\\
\bottomrule
\end{tabular}
\end{center}\end{table*}

\subsection{Skewed EM scheme for SGLD}
The two random sources can be unified into single noise in terms of the first and second moments.
We consider
\begin{equation}
    \rvtheta_{k+1} = \rvtheta_k + \epsilon \left(\nabla\log\pi(\rvtheta_k) + \rvdelta_k \right) + \sqrt{2\epsilon}\rvxi_k.
\end{equation}
with more general noise $\rvdelta$ and $\rvxi$ of
\begin{equation}
    \mathbb{E}[\rvdelta] = \vm_{\delta,1}, \ \ \ 
    \mathbb{E}[\rvdelta^{\circ 2}] = \vm_{\delta,2},
\end{equation}
\begin{equation}
    \mathbb{E}[\rvxi] = \vm_{\xi,1}, \ \ \ 
    \mathbb{E}[\rvxi^{\circ 2}] = \vm_{\xi,2},
\end{equation}
where $\evm_{\delta,1}^d, \evm_{\delta,2}^d, \evm_{\xi,1}^d, \evm_{\xi,2}^d$ are some constants for $d=1,\dots,D$.
The algorithm can be viewed as a skewed EM scheme
\begin{equation}
    \rvtheta_{k+1} = \rvtheta_k + \epsilon \left( \nabla\log\pi(\rvtheta_k) + \rvrho \right),
\end{equation}
where unified gradient noise $\rvrho\in\R^D$ is a random vector satisfying
\begin{equation}\begin{aligned}
    \mathbb{E}[\rvrho] &= \vm_{\rho,1} = \vm_{\delta,1} + \vm_{\xi,1},\\
    \mathbb{E}[\rvrho^{\circ 2}] &= \vm_{\rho,2} = \vm_{\delta,2} + \sqrt{\frac{8}{\epsilon}}\vm_{\delta,1}\vm_{\xi,1} + \frac{2}{\epsilon}\vm_{\xi,2}.    
\end{aligned}\end{equation}
From \autoref{thm:SGD_SLA}, SGLD is regarded as a skewed EM scheme of a continuous-time It\^o process
\begin{equation}
    d\rvtheta_t = \left(\nabla \log \pi(\rvtheta_t) + \vm_{\rho,1}\right) dt + \sqrt{2} \mR d\rvw_t,
\end{equation}
where $\mR = \sqrt{\frac{\epsilon}{2}} \diag{\left( \vm_{\rho,2}-\vm_{\rho,1}^{\circ2} \right)}^\frac{1}{2}$.
Its stationary distribution is
\begin{equation}
    p(\rvtheta) \propto \pi(\rvtheta)^{\frac{\Tr{(\mR^{-2})}}{D}} \exp{\left( \vm_1^\top \mR^{-2} \rvtheta \right)}^\frac{1}{D}.
\end{equation}
This reveals that the SGLD algorithm weakly converges even with biased and varied stochastic gradient.

\subsection{Overestimation of the second moment}
Let us consider the unbiased case of $\vm_{\delta,1} = \vzero$, $\vm_{\xi,1} = \vzero$ and $\vm_{\xi,2} = \vone$.
The second moment of unified noise is always larger than the appropriate scaling (\autoref{cor:appropriate_scaling})
\begin{equation}
    \vm_{\rho,2} = \vm_{\delta,2} + \frac{2}{\epsilon} \vone > \frac{2}{\epsilon}\vone.
\end{equation}
Thus thermodynamic beta of the stationary distribution is lower than one
\begin{equation}
    \beta = {\frac{\Tr{(\mR^{-2})}}{D}} = \frac{2}{D\epsilon}\left( \sum_{d=1}^D\frac{1}{\evm_{\delta,2}^d+\frac{2}{\epsilon}} \right) < 1.
\end{equation}
This means that the SGLD algorithm always drawing samples from a flatter distribution than the target.

Note that an overestimation of Gaussian posterior covariance has been reported by \citet{Mandt:2017aa} under the assumptions of Gaussian gradient noise and quadratic loss function.

\subsection{Avoiding overestimation using adaptive diffusion}
If moments $\vm_{\delta,2} \ (<\vone)$ is estimated, we can employ an adaptive diffusion $\rvxi^*$ of
\begin{equation}
    \mathbb{E}[\rvxi^*] = \vzero, \ \ \ 
    \mathbb{E}[(\rvxi^*)^{\circ 2}] = \vone - \frac{\epsilon}{2} \vm_{\delta,2},
\end{equation}
so that the unified noise has the appropriate moments
\begin{equation}
    \mathbb{E}[\rvrho] = \vzero, \ \ \ \mathbb{E}[\rvrho^{\circ 2}] = \frac{2}{\epsilon}\vone.
\end{equation}
The stationary distribution is the target with $\beta=1$.

Similarly to the existing adaptive moment estimation e.g., RMSprop, AdaDelta, and Adam, we can estimate moments $\E[\widehat{\nabla}\log\pi(\rvtheta)] \approx \vr$ and $\vm_{\delta,2} \approx \vv$ by moving average with some constants $\gamma_1$ and $\gamma_2$
\begin{equation}\begin{aligned}
    \vr_{k+1} &= \gamma_1 \vr_k + (1-\gamma_1) \widehat{\nabla}\log\pi(\theta_k)\\
    \vv_{k+1} &= \gamma_2 \vv_k + (1-\gamma_2) \left(\widehat{\nabla}\log\pi(\theta_k) - \vr_{k+1}\right)^{\circ 2}.
\end{aligned}\end{equation}
Thus the algorithm is
\begin{equation}
    \rvtheta_{k+1} = \rvtheta_k + \epsilon \widehat{\nabla}\log\pi(\rvtheta_k) + \sqrt{2\epsilon} \widehat{\rvxi}_k^*
\end{equation}
with adaptive diffusion
\begin{equation}
    \mathbb{E}[\widehat{\rvxi}^*] = \vzero, \ \ \ 
    \mathbb{E}[(\widehat{\rvxi}^*)^{\circ 2}] = \vone - \frac{\epsilon}{2}\vv.
\end{equation}
Such a component random variable can be efficiently generated by a two-point distributed random variable
\begin{equation}
    P\left((\widehat{\ervxi}^*)^d = \pm\sqrt{1 - \frac{\epsilon}{2}v^d}\right) = 0.5.
\end{equation}

This extension is straightforwardly derived because the skewed EM scheme does not have the third or higher moment condition.
It is basically difficult to design the diffusion to standardize the third or higher moments.

\section{Conclusion}
We introduced the skewed EM scheme to construct a common foundation of SGD and SGLD.
It relaxes the existing assumptions on gradient noise and loss functions, extending possible stationary distributions to the Gibbs distribution of a general loss function.
We validated it by the weak approximation theory in stochastic analysis and numerical experiments.
We investigate the convergence of SGD with biased gradient and reinterpreted SGLD to propose an extension to avoid overestimating the second moment of diffusion.

\clearpage
\appendix

\section{Proof of \autoref{thm:main}}\label{app:proof_thm:main}
Let one-step difference of the continuous-time dynamics be $\Delta$, the EM scheme be $\widetilde{\Delta}$, and the skewed EM scheme be $\overline{\Delta}$ from given value $\vx_{t_{k}} = \vy_{k} \in \R^D$, so that
\begin{equation}\begin{aligned}\label{equ:deltas}
    \Delta &= \rvx_{t_{k+1}} - \vx_{t_k}\\
    \widetilde{\Delta} &= \epsilon \va(\vy_k) + \sqrt{\epsilon} \mB(\vy_k)\rvzeta_k\\
    \overline{\Delta} &= \epsilon \va(\vy_k) + \sqrt{\epsilon} \mB(\vy_k)\rvxi_k,
\end{aligned}\end{equation}
where 
\begin{equation}
    \mathbb{E}[(\ervzeta^m)^p] = \begin{cases}
        0       & (p=1,3,\dots)\\
        (p-1)!! & (p=2,4,\dots)
    \end{cases}
\end{equation}
and 
\begin{equation}
    \mathbb{E}[(\ervxi^m)^p] = \begin{cases}
        0   & (p = 1)\\
        1   & (p = 2).
    \end{cases}
\end{equation}
The $d$-th components are denoted by $\Delta^d, \widetilde{\Delta}^d, \overline{\Delta}^d$.

Let us review the following known results of the EM scheme before our proof.
\begin{lemma}[\citep{Milstein:1979aa}]\label{lem:onestep_lemma1_EM_a}
    There exists some constants $C>0$, $\kappa>0$ such that
    \begin{equation}
        \left| \E\left[ \Delta^d - \widetilde{\Delta}^d \right] \right| \leq C(1+\|\rvy_k\|_2^\kappa) \epsilon^2
    \end{equation}
    for $d=1,\ldots,D$ and $0<\epsilon<1$.
\end{lemma}
\begin{lemma}[\citep{Milstein:1979aa}]\label{lem:onestep_lemma1_EM_b}
    There exists some constants $C>0$, $\kappa>0$ such that
    \begin{equation}
        \left| \E\left[\Delta^{d_1}\Delta^{d_2} - \widetilde{\Delta}^{d_1}\widetilde{\Delta}^{d_2} \right] \right| \leq C(1+\|\rvy_k\|_2^\kappa) \epsilon^2
    \end{equation}
    for $d_1,d_2=1,\ldots,D$ and $0<\epsilon<1$.
\end{lemma}

We often employ the linear growth condition of coefficients $\va$ and $\mB$.
\begin{lemma}[Linear growth condition]
    For all $\vx\in\R^D$ the following inequality holds for some constant $C>0$
    \begin{equation}
        \|\va(\vx)\|_2 + \sum_{m=1}^M \|\emB^{*,m}(\vx)\|_2 \leq C (1+\|\vx\|_2).
    \end{equation}
\end{lemma}
\begin{proof}
    It is directly derived by taking $\vy=\vzero$ in the Lipschitz condition (\autoref{equ:Lipschitz}).
\end{proof}

We then check the following properties of the skewed EM scheme for one-step approximation.
Our approach is to show that the skewed EM scheme closely approximates the EM scheme with respect to its first and second moments.
\begin{lemma}\label{lem:onestep_lemma1_a}
    There exists some constants $C>0$, $\kappa>0$ such that
    \begin{equation}
        \left| \E\left[ \Delta^d - \overline{\Delta}^d \right] \right| \leq C(1+\|\rvy_k\|_2^\kappa) \epsilon^2
    \end{equation}
    for $d=1,\ldots,D$ and $0<\epsilon<1$.
\end{lemma}
\begin{proof}
    From the triangle inequality of $\E\left[ \Delta^d - \widetilde{\Delta}^d \right]$ and $\E\left[ \widetilde{\Delta}^d - \overline{\Delta}^d \right]$, we have
    \begin{equation}
        \left| \E\left[ \Delta^d - \overline{\Delta}^d \right] \right| \leq \left|\E\left[ \Delta^d - \widetilde{\Delta}^d \right]\right| + \left|\E\left[ \widetilde{\Delta}^d - \overline{\Delta}^d \right]\right|.
    \end{equation}
    From \autoref{lem:onestep_lemma1_EM_a}, our goal is to show
    \begin{equation}\label{equ:lem_m1_goal}
        \left|\E\left[ \widetilde{\Delta}^d - \overline{\Delta}^d \right]\right| \leq C(1+\|\rvy_k\|_2^\kappa) \epsilon^2.
    \end{equation}
    From \autoref{equ:deltas},
    \begin{equation}
        \left|\E\left[ \widetilde{\Delta}^d - \overline{\Delta}^d \right]\right|
        = \left|\E\left[ \mB^{d,*}(\vy_k) (\rvzeta_k - \rvxi_k) \right]\right| \sqrt{\epsilon}
    \end{equation}
    Because $\E[\ervzeta^m] = \E[\ervxi^m] = 0$,
    \begin{equation}
        \left|\E\left[ \widetilde{\Delta}^d - \overline{\Delta}^d \right]\right| = 0.
    \end{equation}
    This satisfies \autoref{equ:lem_m1_goal}.
\end{proof}

\begin{lemma}\label{lem:onestep_lemma1_b}
    There exists some constants $C>0$, $\kappa>0$ such that
    \begin{equation}
        \left| \E\left[\Delta^{d_1}\Delta^{d_2} - \overline{\Delta}^{d_1}\overline{\Delta}^{d_2} \right] \right| \leq C(1+\|\rvy_k\|_2^\kappa) \epsilon^2
    \end{equation}
    for $d_1,d_2=1,\ldots,D$ and $0<\epsilon<1$.
\end{lemma}
\begin{proof}
    From the triangle inequality of $\E\left[ \Delta^{d_1}\Delta^{d_2} - \widetilde{\Delta}^{d_1}\widetilde{\Delta}^{d_2} \right]$ and $\E\left[ \widetilde{\Delta}^{d_1}\widetilde{\Delta}^{d_2} - \overline{\Delta}^{d_1}\overline{\Delta}^{d_2} \right]$, we have
    \begin{equation}\begin{aligned}
        \left|\E\left[ \Delta^{d_1}\Delta^{d_2} - \overline{\Delta}^{d_1}\overline{\Delta}^{d_2} \right]\right| &\leq \left|\E\left[ \Delta^{d_1}\Delta^{d_2} - \widetilde{\Delta}^{d_1}\widetilde{\Delta}^{d_2} \right]\right| + \left|\E\left[ \widetilde{\Delta}^{d_1}\widetilde{\Delta}^{d_2} - \overline{\Delta}^{d_1}\overline{\Delta}^{d_2} \right]\right|.
    \end{aligned}\end{equation}
    From \autoref{lem:onestep_lemma1_EM_b}, our goal is to show
    \begin{equation}\label{equ:lem_m2_goal}
        \left|\E\left[ \widetilde{\Delta}^{d_1}\widetilde{\Delta}^{d_2} - \overline{\Delta}^{d_1}\overline{\Delta}^{d_2} \right]\right| \leq C(1+\|\rvy_k\|_2^\kappa) \epsilon^2.
    \end{equation}
    From \autoref{equ:deltas},
    \begin{equation}\begin{aligned}
        \widetilde{\Delta}^{d_1}\widetilde{\Delta}^{d_2} &= \epsilon^2 a^{d_1}a^{d_2} + \epsilon^{1.5} a^{d_1}\emB^{d_2,*}\rvzeta_k + \epsilon^{1.5}a^{d_2}\emB^{d_1,*}\rvzeta_k + \epsilon \emB^{d_1,*}\rvzeta_k\emB^{d_2,*}\rvzeta_k,
    \end{aligned}\end{equation}
    where functions $\va, \mB$ are evaluated at $\rvy_k$.
    Thus
    \begin{equation}\begin{aligned}
        \E\left[ \widetilde{\Delta}^{d_1}\widetilde{\Delta}^{d_2} \right] &= \epsilon^2 a^{d_1}a^{d_2} + \epsilon^{1.5} a^{d_1}\emB^{d_2,*}\E[\rvzeta_k] + \epsilon^{1.5}a^{d_2}\emB^{d_1,*}\E[\rvzeta_k] + \epsilon \sum_{m,n=1}^M(\emB^{d_1,m}\emB^{d_2,n}\E[\ervzeta_k^m\ervzeta_k^n])\\
    \end{aligned}\end{equation}
    Because $\E[\rvzeta] = \vzero$ and 
    \begin{equation}
        \E[\ervzeta^m\ervzeta^n] = \begin{cases}
            \E[\ervzeta^m]\E[\ervzeta^n] &= 0 \ \ \ (m \neq n)\\
            \E[(\ervzeta^m)^2] &= 1 \ \ \ (m = n),
        \end{cases}
    \end{equation}
    we have
    \begin{equation}
        \E\left[ \widetilde{\Delta}^{d_1}\widetilde{\Delta}^{d_2} \right] = \epsilon^2 a^{d_1}a^{d_2} + \epsilon \sum_{m=1}^M \emB^{d_1,m}\emB^{d_2,m}.
    \end{equation}
    In the same way, we obtain
    \begin{equation}\begin{aligned}
        \E\left[ \overline{\Delta}^{d_1}\overline{\Delta}^{d_2} \right] &= \epsilon^2 a^{d_1}a^{d_2} + \epsilon^{1.5} a^{d_1}\emB^{d_2,*}\E[\rvxi_k] + \epsilon^{1.5}a^{d_2}\emB^{d_1,*}\E[\rvxi_k] + \epsilon \sum_{m,n=1}^M(\emB^{d_1,m}\emB^{d_2,n}\E[\ervxi_k^m\ervxi_k^n])\\
    \end{aligned}\end{equation}
    Because $\E[\rvxi] = \vzero$ and 
    \begin{equation}
        \E[\ervxi^m\ervxi^n] = \begin{cases}
            \E[\ervxi^m]\E[\ervxi^n] &= 0 \ \ \ (m \neq n)\\
            \E[(\ervxi^m)^2] &= 1 \ \ \ (m = n),
        \end{cases}
    \end{equation}
    we have
    \begin{equation}
        \E\left[ \overline{\Delta}^{d_1}\overline{\Delta}^{d_2} \right] = \epsilon^2 a^{d_1}a^{d_2} + \epsilon \sum_{m=1}^M \emB^{d_1,m}\emB^{d_2,m}.
    \end{equation}
    Therefore
    \begin{equation}
        \left|\E\left[ \widetilde{\Delta}^{d_1}\widetilde{\Delta}^{d_2} - \overline{\Delta}^{d_1}\overline{\Delta}^{d_2} \right]\right| = 0.
    \end{equation}
    This satisfies \autoref{equ:lem_m2_goal}.
\end{proof}

\begin{lemma}\label{lem:onestep_lemma2}
    There exists some constants $C>0$, $\kappa>0$ such that
    \begin{equation}\label{equ:onestep_asm2}
        \E\left[ \left|\overline{\Delta}^{d_1}\right| \left|\overline{\Delta}^{d_2}\right| \left|\overline{\Delta}^{d_3}\right| \right] \leq C(1+\|\rvy_k\|_2^\kappa) \epsilon^{1.5}
    \end{equation}
    for $d_1,d_2,d_3=1,\ldots,D$ and $0<\epsilon<1$.
\end{lemma}
\begin{proof}
    \begin{equation}\begin{aligned}
        \E&\left[ \left|\overline{\Delta}^{d_1}\right| \left|\overline{\Delta}^{d_2}\right| \left|\overline{\Delta}^{d_3}\right| \right] = \E\left[ \left|\overline{\Delta}^{d_1} \overline{\Delta}^{d_2} \overline{\Delta}^{d_3}\right| \right]\\
        &= \E[| \epsilon^3a^{d_1}a^{d_2}a^{d_3} + \epsilon^{2.5}a^{d_1}a^{d_2}\emB^{d_3,*}\rvxi + \epsilon^{2.5}a^{d_2}a^{d_3}\emB^{d_1,*}\rvxi + \epsilon^{2.5}a^{d_3}a^{d_1}\emB^{d_2,*}\rvxi\\
        &\ \ + \epsilon^2a^{d_1}\emB^{d_2,*}\rvxi\emB^{d_3,*}\rvxi + \epsilon^2a^{d_2}\emB^{d_3,*}\rvxi\emB^{d_1,*}\rvxi + \epsilon^2a^{d_3}\emB^{d_1,*}\rvxi\emB^{d_2,*}\rvxi + \epsilon^{1.5}\emB^{d_1,*}\rvxi\emB^{d_2,*}\rvxi\emB^{d_3,*}\rvxi |].
    \end{aligned}\end{equation}
    From the triangle inequality,
    \begin{equation}\begin{aligned}
        \E\left[ \left|\overline{\Delta}^{d_1}\right| \left|\overline{\Delta}^{d_2}\right| \left|\overline{\Delta}^{d_3}\right| \right] &\leq \epsilon^{1.5}\E\left[\left|\emB^{d_1,*}\rvxi\emB^{d_2,*}\rvxi\emB^{d_3,*}\rvxi\right|\right] + \mathcal{O}(\epsilon^2)\\
        &\leq C\epsilon^{1.5}\E\left[\left| \sum_{m,n,o=1}^M\left(\emB^{d_1,m}\emB^{d_2,n}\emB^{d_3,o}\ervxi^m\ervxi^n\ervxi^o\right) \right|\right],
    \end{aligned}\end{equation}
    where $C>0$ is some constant. From the triangle inequality,
    \begin{equation}\begin{aligned}
        \E\left[ \left|\overline{\Delta}^{d_1}\right| \left|\overline{\Delta}^{d_2}\right| \left|\overline{\Delta}^{d_3}\right| \right] &\leq C\epsilon^{1.5}\E\left[ \sum_{m,n,o=1}^M\left|\emB^{d_1,m}\emB^{d_2,n}\emB^{d_3,o}\ervxi^m\ervxi^n\ervxi^o\right| \right]\\
        &= C\epsilon^{1.5} \sum_{m,n,o=1}^M \left|\emB^{d_1,m}\emB^{d_2,n}\emB^{d_3,o}\right| \E\left[\left| \ervxi^m\ervxi^n\ervxi^o \right|\right]\\
    \end{aligned}\end{equation}
    From the linear growth of $\mB^{d,m}$ and finiteness of $\E\left[\left| \ervxi^m \right|\right]$, $\E\left[\left| (\ervxi^m)^2 \right|\right]$, and $\E\left[\left| (\ervxi^m)^3 \right|\right]$,
    \begin{equation}
        \E\left[ \left|\overline{\Delta}^{d_1}\right| \left|\overline{\Delta}^{d_2}\right| \left|\overline{\Delta}^{d_3}\right| \right] \leq C(1+\|\rvy_k\|_2^3) \epsilon^{1.5}.
    \end{equation}
    This satisfies \autoref{equ:onestep_asm2}.
\end{proof}

The following lemmas are sufficient conditions for $\E[\rvy_k+\overline{\Delta}]^{2l}, \ (l\in\mathbb{N})$ to exist and be bounded.
\begin{lemma}\label{lem:onestep_lemma3}
    There exists some constant $C>0$ such that
    \begin{equation}\label{equ:onestep_asm3}
        \left\|\E\left[\overline{\Delta}\right]\right\|_2 \leq C(1+\|\rvy_k\|_2)\epsilon,
    \end{equation}
    for $0<\epsilon<1$.
\end{lemma}
\begin{proof}
    From $\E[\rvxi] = \vzero$,
    \begin{equation}\begin{aligned}
        \left\|\E\left[\overline{\Delta}\right]\right\|_2 &= \left\|\E\left[\epsilon\va(\rvy_k) + \sqrt{\epsilon}\mB(\rvy_k)\rvxi_k\right]\right\|_2\\
        &= \left\|\va(\rvy_k)\right\|_2 \epsilon\\
    \end{aligned}\end{equation}
    From the linear growth condition of $\va$,
    \begin{equation}\begin{aligned}
        \left\|\E\left[\overline{\Delta}\right]\right\|_2 \leq C (1+\|\rvy_k\|_2) \epsilon.
    \end{aligned}\end{equation}
\end{proof}

\begin{lemma}\label{lem:onestep_lemma4}
    There exists $M(\rvxi)$ having moments of all orders such that
    \begin{equation}\label{equ:onestep_asm4}
        \left\|\overline{\Delta}\right\|_2 \leq M(\rvxi_k)(1+\|\rvy_k\|_2)\sqrt{\epsilon}
    \end{equation}
    for $0<\epsilon<1$.
\end{lemma}
\begin{proof}
    From the triangle inequality,
    \begin{equation}\begin{aligned}
        \left\|\overline{\Delta}\right\|_2 &= \left\| \epsilon\va(\rvy_k) + \sqrt{\epsilon}\mB(\rvy_k)\rvxi_k \right\|_2\\
        &\leq \epsilon \left\| \va(\rvy_k) \right\|_2 + \sqrt{\epsilon} \sum_{m=1}^M \left\| \emB^{*,m}(\rvy_k) \right\|_2 |\rvxi_k^m|.
    \end{aligned}\end{equation}
    From the linear growth condition of $\mB$, 
    \begin{equation}\begin{aligned}
        \left\|\overline{\Delta}\right\|_2 &\leq \sqrt{\epsilon} \sum_{m=1}^M\left( C_m (1+\|\rvy_k\|_2) |\rvxi_k^m| \right)\\
        &= M(\rvxi_k) (1+\|\rvy_k\|_2) \sqrt{\epsilon},
    \end{aligned}\end{equation}
    where $M(\rvxi_k) = \sum_{m=1}^M C_m|\ervxi_k^m|$.
\end{proof}

Then our proof follows Theorem 2 and Lemma 5 of \citet{Milstein:1986aa} that connects one-step approximation to approximation on a finite interval.
\begin{theorem}[Theorem 2 and Lemma 5 of \citep{Milstein:1986aa} (weak order 0.5)]
    Suppose Assumption\,\ref{asm:common}, \ref{asm:milstein}, and Lemma\,\ref{lem:onestep_lemma1_a}, \ref{lem:onestep_lemma1_b}, \ref{lem:onestep_lemma2}, \ref{lem:onestep_lemma3}, and \ref{lem:onestep_lemma4}.
    Then, for all $K>0$ and all $k=0,1,\dots,K$ \Eqref{equ:skewedEM} has order of accuracy $0.5$.
\end{theorem}

\section{Proof of \autoref{lem:impossibleEM}}\label{app:proof_lem:impossibleEM}
Let one-step difference of the \autoref{equ:impossibleEM} be
\begin{equation}
    \underline{\Delta} = \epsilon \va(\vy_k) + \sqrt{\epsilon} \mB(\vy_k)\rvpsi_k,
\end{equation}
where $\E[\rvpsi] = \vzero$.

From a similar derivation to \autoref{lem:onestep_lemma1_b}, the following lemma holds:
\begin{lemma}\label{lem:impossible_onestep}
    There exists some constants $C>0$, $\kappa>0$ such that
    \begin{equation}
        \left| \E\left[\Delta^{d_1}\Delta^{d_2} - \underline{\Delta}^{d_1}\underline{\Delta}^{d_2} \right] \right| \leq C(1+\|\rvy_k\|_2^\kappa) \epsilon
    \end{equation}
    for $d_1,d_2=1,\ldots,D$ and $0<\epsilon<1$.
\end{lemma}

Then from Theorem 2 and Lemma 5 of \citet{Milstein:1986aa}, we have \autoref{lem:impossibleEM}.
\begin{theorem}[Theorem 2 and Lemma 5 of \citep{Milstein:1986aa} (weak order 0)]
    Suppose Assumption\,\ref{asm:common}, \ref{asm:milstein}, and Lemma\,\ref{lem:onestep_lemma1_a}, \ref{lem:onestep_lemma2}, \ref{lem:onestep_lemma3}, \ref{lem:onestep_lemma4}, \ref{lem:impossible_onestep}.
    Then, for all $K>0$ and all $k=0,1,\dots,K$ \Eqref{equ:impossibleEM} has order of accuracy $0$.
\end{theorem}

\section{Proof of \autoref{thm:SGD_SLA}}\label{app:proof_thm:SGD_SLA}
\begin{proof}
    Decompose noise as
    \begin{equation}
        \rvdelta = \vm_1 + \diag{\left(\vm_2-\vm_1^{\circ2}\right)}^\frac{1}{2} \rvxi,
    \end{equation}
    where random vector $\rvxi$ with component $\ervxi^d$ satisfying $\E[\ervxi^d]=0$ and $\E[(\ervxi^d)^2]=1$ for $d=1,\dots,D$.
    Thus
    \begin{equation}\begin{aligned}\label{equ:SLASGD_theta}
        \rvtheta_{k+1} &= \rvtheta_k + \epsilon \widehat{\nabla}\log\exp(- L(\rvtheta_k))\\
        &= \rvtheta_k + \epsilon (\nabla\log\pi(\rvtheta_k)+\vm_1) + \sqrt{2\epsilon} \mR\rvxi_k,
    \end{aligned}\end{equation}
    where $\mR = \sqrt{\frac{\epsilon}{2}} \diag{\left(\vm_2-\vm_1^{\circ2}\right)}^\frac{1}{2}$.
    From \autoref{thm:main}, \autoref{equ:SLASGD_theta} is a weak order 0.5 time-discrete approximation of 
    \begin{equation}
        d\rvtheta_t = \left(\nabla \log \pi(\rvtheta_t) + \vm_1\right) dt + \sqrt{2} \mR d\rvw_t.
    \end{equation}
    
    Let random vector $\rvvarphi = \mR^{-1} \rvtheta$. 
    From the multivariate Lamperti transform\,\citep{Moller:2010aa}, $\rvvarphi_t$ is also an It\^o process
    \begin{equation}
        d\rvvarphi_t = \mR^{-1}(\nabla_\rvtheta\log\pi(\rvtheta) + \vm_1)dt + \sqrt{2} d\rvw_t
    \end{equation}
    that has a normalized diffusion coefficient.
    Here we distinguish gradient operators $\nabla_\rvtheta=(\frac{\partial}{\partial \ervtheta_1}, \dots, \frac{\partial}{\partial \ervtheta_D})^\top$ and $\nabla_\rvvarphi=(\frac{\partial}{\partial \ervvarphi_1}, \dots, \frac{\partial}{\partial \ervvarphi_D})^\top$.
    Let $p(\rvvarphi)$ be the stationary distribution of $\rvvarphi$.
    From the well known relationship
    \begin{equation}
        \nabla_\rvvarphi \log p(\rvvarphi) = \mR^{-1}(\nabla_\rvtheta\log\pi(\rvtheta) + \vm_1),
    \end{equation}
    therefore
    \begin{equation}\begin{aligned}
        D\log p(\rvvarphi) &= \int \mR^{-1}(\nabla_\rvtheta\log\pi(\rvtheta) + \vm_1) \cdot d\rvvarphi\\
        &= \int \mR^{-2}\nabla_\rvvarphi\log\pi(\mR\rvvarphi) \cdot d\rvvarphi + \int \mR^{-1}\vm_1 \cdot d\rvvarphi\\
        &= \Tr{(\mR^{-2})} \log\pi(\mR\rvvarphi) + (\mR^{-1}\vm_1)\cdot\rvvarphi,
    \end{aligned}\end{equation}
    we have 
    \begin{equation}
        p(\rvvarphi) = \pi(\mR\rvvarphi)^{\frac{\Tr{(\mR^{-2})}}{D}} \exp{((\mR^{-1}\vm_1)\cdot\rvvarphi)}^\frac{1}{D}.
    \end{equation}
    
    From the change of the random vectors
    \begin{equation}
        p(\rvtheta) = p(\rvvarphi) \left|\det \frac{d}{d\rvtheta} \rvvarphi \right| = p(\rvvarphi) \left|\det\mR\right|,
    \end{equation}
    we obtain
    \begin{equation}\begin{aligned}
        p(\rvtheta) &\propto \pi(\rvtheta)^{\frac{\Tr{(\mR^{-2})}}{D}} \exp{\left( (\mR^{-1}\vm_1) \cdot (\mR^{-1}\rvtheta) \right)}^\frac{1}{D}\\
        &= \pi(\rvtheta)^{\frac{\Tr{(\mR^{-2})}}{D}} \exp{\left( \vm_1^\top \mR^{-2} \rvtheta \right)}^\frac{1}{D}.
    \end{aligned}\end{equation}
\end{proof}

\clearpage

\bibliography{ref}
\bibliographystyle{plainnat}

\end{document}